\documentclass[runningheads]{llncs}
\usepackage[T1]{fontenc}
\usepackage{graphicx}
\usepackage{amsmath,amssymb}
\usepackage{mathtools}
\usepackage{enumerate}
\usepackage{xcolor}
\usepackage{stmaryrd} % Where \llbracket and \rrbracket are defined
\usepackage{xspace}
\usepackage{tikz}
\usepackage{pgfplots}
\usetikzlibrary{automata,positioning,arrows.meta,calc,decorations.shapes, shapes.geometric}
\usepackage{subcaption}
\usepackage{array}
\usepackage{booktabs}
\usepackage{wrapfig}

\usepackage{algorithm,algpseudocode,algorithmicx}

\algdef{SE}[DOWHILE]{Do}{doWhile}{\algorithmicdo}[1]{\algorithmicwhile\ #1}%

\newcommand{\tup}[1]{\ensuremath{\left(#1\right)}}
\newcommand{\set}[1]{\ensuremath{\left\lbrace #1\right\rbrace}}
\newcommand{\pref}{\mathit{Pref}}
\mathchardef\mhyphen="2D

\newcommand{\Sys}{\Sigma}
\newcommand{\Xs}{\mathcal{X}}
\newcommand{\x}{x}

\newcommand{\Us}{\mathcal{U}}
\renewcommand{\u}{u}
\newcommand{\Ws}{\mathcal{W}}
\newcommand{\w}{w}
\newcommand{\f}{f}
\newcommand{\tran}[1]{\xrightarrow{#1}}
\newcommand{\traj}{\xi}
\newcommand{\path}{\rho}

\newcommand{\Cont}{\mathcal{C}}
\newcommand{\Contmax}{\Cont^*}
\newcommand{\paths}[3]{\mathit{Paths}(#1,#2,#3)}
\newcommand{\dom}[1]{\mathit{Dom}(#1)}

\newcommand{\Contfc}{\overline{\mathcal{C}}}

\newcommand{\abs}[1]{\widehat{#1}}
\newcommand{\Sysh}{\abs{\Sigma}}
\newcommand{\Xsh}{\abs{\mathcal{X}}}
\newcommand{\xh}{\abs{x}}

\newcommand{\Ush}{\abs{\mathcal{U}}}
\newcommand{\uh}{\abs{u}}

\newcommand{\fh}{\abs{f}}
\newcommand{\Phih}{\abs{\Phi}}
\newcommand{\Conth}{\abs{\Cont}}

\newcommand{\Ghi}{\abs{G_i}}

\newcommand{\cpre}{\mathit{CPre}}

\newcommand{\Shield}{\mathcal{S}}

\newcommand{\safety}{\mathit{Safety}}
\newcommand{\psafety}{\mathrm{P}\text{-}\mathit{Safety}}

\newcommand{\SafetySynt}{\texttt{SafetyControl}\xspace}
\newcommand{\NBSynt}{\texttt{NBControl}\xspace}
\newcommand{\nmp}{n.m.p.\xspace}

\begin{document}
\title{Efficient Dynamic Shielding\\ for Parametric Safety Specifications}
\author{Davide Corsi \inst{1} \and
Kaushik Mallik \inst{2} \and
Andoni Rodr\'iguez \inst{2,3} \and
C\'esar S\'anchez \inst{2}}
\institute{University of California, Irvine, USA\\
\email{dcorsi@uci.edu}\and
IMDEA Software Institute, Spain\\\and
Universidad Polit\'ecnica de Madrid, Spain\\
\email{\{kaushik.mallik,andoni.rodriguez,cesar.sanchez\}@imdea.org}}
\maketitle              % typeset the header of the contribution
\begin{abstract}
Shielding has emerged as a promising approach for ensuring safety of AI-controlled autonomous systems.
The algorithmic goal is to compute a shield, which is a runtime safety enforcement tool that needs to monitor and intervene the AI controller's actions if safety could be compromised otherwise.
Traditional shields are  designed statically for a specific safety requirement.
Therefore, if the safety requirement changes at runtime due to changing operating conditions, the shield needs to be recomputed from scratch, causing delays that could be fatal.
We introduce \emph{dynamic shields} for \emph{parametric} safety specifications, which are succinctly represented sets of all possible safety specifications that may be encountered at runtime.
Our dynamic shields are statically designed for a given safety parameter set, and are able to dynamically adapt as the true safety specification (permissible by the parameters) is revealed  at runtime.
The main algorithmic novelty lies in the dynamic adaptation procedure, which is a simple and fast algorithm that utilizes known features of standard safety shields, like maximal permissiveness.
We report experimental results for a robot navigation problem in unknown territories, where the safety specification evolves as new obstacles are discovered at runtime.
In our experiments, the dynamic shields took a few minutes for their offline design, and took between a fraction of a second and a few seconds for online adaptation at each step, whereas the brute-force online recomputation approach was up to 5 times slower.

\end{abstract}

\section{Introduction}

Most critical autonomous systems like self-driving cars are nowadays controlled by machine-learned (ML) controllers, and ensuring their safety is an important agenda in artificial intelligence and formal methods research.
Despite the recent advancements in the formal verification of ML controllers~\cite{katz17reluplex,ivanov2019verisig,amir21towards}, most methods either do not scale to the size and complexity of the ML components, or they do not support infinite-horizon safety properties~\cite{amir21towards}.
Moreover, all these verification approaches assume knowledge of the ML components' parameters, which might not be true if the controller is designed by a third party and appears as a black-box.
One promising solution is \emph{shielding}~\cite{bloem2015shield,alshiekh2018safe,bharadwaj2019synthesis,li2020robust}, where we deploy a formally verified runtime enforcement tool, called the \emph{shield}, that monitors---and, if necessary, overrides---the actions of an ML controller during deployment.
Usually shield synthesis is cheaper than verifying the entire system, because the synthesis happens on a small system abstraction that concerns only the safety aspects. 
Furthermore, the synthesis process treats the ML controller as a black-box, thereby bypassing the aforementioned shortcomings of the traditional verification approaches.
Shielding has been successfully applied in tandem with complex machine-learned controllers in a large variety of applications, including safe human-robot interactions~\cite{hu2022sharp} and safe autonomous driving~\cite{raeesi2025safe}.

State-of-the-art shield synthesis algorithms offer \emph{statically} designed shields, crafted for a specific safety objective provided at the design time.
In reality, safety specifications often vary over time, and there are no principled approaches to \emph{dynamically} adapt a (statically designed) shield as new safety objectives are uncovered at runtime.
This type of dynamic safety specifications appears naturally in designing navigation controllers for robots in previously unknown workspaces, which is the motivating use case of our work and is described as follows.
We are given a robot in a workspace filled with static obstacles, and the robot must always avoid colliding with them.
However, the visibility of the robot is limited by the range of its sensors, and therefore it can see the obstacles only when it gets close to them.
%The underlying ML controller with undisclosed motives drives the robot around, and the shield must intervene whenever the ML controller is about to use an unsafe action.
If the entire map were visible to the shield, it could use the locations of the obstacles to define its static safety specification.
However, due to its limited sensing power, the safety specification would only concern its immediate neighborhood, and would keep changing in real-time as it moves.

We present a novel framework of \emph{dynamic shielding} with respect to evolving safety specifications.
We assume that we are given a perturbed, discrete-time dynamical model of the system, and a parameterized (finite) set of all possible safety specifications that could be encountered at runtime.
Specifically, as the specification, we are given a finite collection of safety objectives of the form $\set{\Box\, G_i}_{i}$, called the \emph{parameter set}, where $G_i$ is a set of safe states of the system, and $\Box\, G_i$ is a linear temporal logic (LTL) formula specifying that $G_i$ must not be left at any time.
Each formula $\Box\, G_i$ represents an \emph{atomic} safety objective, and the actual safety specification encountered at runtime will be the conjunction of an arbitrary subset of atomic safety objectives.
The aim is to statically design a shield for the statically provided parameter set $\set{\Box\, G_i}_{i}$, such that the shield can dynamically adapt itself for every dynamically generated safety specification.

Naturally, shielding against parametric safety would require coordination between the offline and online design phases, and the two extreme ends of the coordination spectrum are as follows.
The pure offline approach would design one (static) shield for each subset of $\set{\Box G_i}_i$, so that the right shield could be instantly deployed at runtime.
However, this would require solving an exponential number of offline shield synthesis problems which will not scale if the parameter set is large.
In contrast, the pure online approach would perform no computation in the offline phase, and at runtime, whenever a new safety specification is revealed, it would compute the corresponding (static) shield at that moment.
This would create delays in deploying the shield, and could be fatal in most cases.

We present an efficient dynamic shielding algorithm that creates a harmony between the offline and online design phases.
In the offline design phase, one (static) \emph{atomic} shield $\Shield_i$ is computed for each atomic safety specification $\Box G_i$, solving a linear number of shield synthesis problems as opposed to the exponential case of the pure offline algorithm.
In the online deployment phase, as a new safety specification $\Phi = \Box\, G_j\land \Box\, G_{k} \land \ldots$ is encountered, the respective atomic shields $\Shield_j,\Shield_k,\ldots$ are \emph{composed} to obtain the shield for the specification $\Phi$.
%This composition operation is the main technical novelty of this paper.
The composition operation utilizes simple known features like maximal permissiveness of safety shields, giving rise to a fast composition algorithm involving shield ``intersections'' followed by iterative deadlock removals.
Similar style of composition operations are known in the literature in the context of the richer class of parity specifications~\cite{anand2023synthesizing}, though their use in the specialized case of (dynamic) safety shielding is novel.
In summary, our composition algorithm leads to a lightweight shield adaptation mechanism that is significantly cheaper than computing a new shield for $\Phi$ from scratch (the na\"ive pure online approach).

We propose abstraction-based synthesis algorithms for our dynamic shields, though other alternatives could also be pursued~\cite{vidal2000controlled}.
Concretely, we first create an abstract model of the system by following standard procedure~\cite{reissig2016feedback}, namely discretizing the system's state and input space using uniform grids, and then conservatively approximating the system dynamics over the discrete spaces.
Afterwards, we adapt existing abstraction-based synthesis algorithms~\cite{reissig2016feedback} for the offline design and online adaptation of our dynamic shields.
These procedures are compatible with symbolic data structures, particularly binary decision diagrams (BDD), giving rise to efficient implementation of our dynamic shields.

We also provide practical strategies to address the following \emph{safe handover} question that naturally arises:
If the safety specification evolves at each step, it may not be the case that the current shield's actions would keep the system within the domain of the next shield.
How to guarantee safety in such cases?
We address this question for the specific problem of safe robot navigation in unknown territories, where the shield may encounter previously unknown (stationary) obstacles from time to time.
We propose the most conservative solution to the safe handover problem, namely at each step, the shield needs to assume that the entire unobservable state space is unsafe, guaranteeing infinite-horizon safety that also keeps the system \emph{inside} the current observable region.
This shield is continued to be used until it is possible to find a new shield, based on the new surrounding, whose domain contains the current state of the robot.
We prove that our safe handover mechanism will guarantee safety and non-blockingness, i.e., some safe actions will always be available.

We demonstrate the practical effectiveness and feasibility of the dynamic shields using a prototype implementation within the tool Mascot-SDS~\cite{majumdar2023flexible}.
The safety rate of the shields were $100\%$, which is unsurprising since they are correct by construction.
The offline design of our dynamic shields ended within minutes, and the online adaption per step on an average took between  a fraction of a second upto a few seconds, which was upto $5$ times faster compared to the pure online baseline.
This demonstrates the practical feasibility of dynamic shields.

In short, our contributions are as follows:
\begin{enumerate}[(a)]
	\item We propose the dynamic shielding framework for parametric safety specifications, orchestrating offline atomic shield synthesis with a lightweight online adaptation procedure.
	\item We show how our algorithms can be symbolically implemented using the abstraction-based control paradigm.
	\item We present a practical approach to address the safe handover question for dynamic shields in navigation tasks in unknown territories.
	\item We present the superior computational performance of our dynamic shields using a prototype implementation.
\end{enumerate}
\subsection*{Related Works}

Shielding has become one of the enabling technologies in guaranteeing safety of arbitrarily complex machine-learned controllers in autonomous systems~\cite{konighofer2017shield,alshiekh2018safe,bharadwaj2019synthesis,li2020robust,elsayed2021safe,banerjee24dynamic}.
It has been studied in two operational settings, namely pre-shielding and post-shielding.
In pre-shielding, shields are used to guarantee safety during the learning process, while in post-shielding, shields are used during deployment.
While most classical works considered the qualitative safety specifications, recent works have developed shields for quantitative safety specifications~\cite{wu19shield,corsi24verification,rodriguez25shield,kim25realizable,rodriguez25explanations}.
Our dynamic shields consider qualitative safety specifications, which makes them usable either as a pre-shield or as a post-shield~\cite{pranger2021tempest}.

Recent works have considered various different types of dynamic shielding frameworks.
When the underlying model parameters like environment probability distributions are unknown or partially known, shields will need to dynamically adapt to account for newly discovered model parameters~\cite{pranger2021adaptive,waga2022dynamic,chi23safe,carr23safe,debot25neurosymbolic,feng2025adaptive}.
When shields' objectives are quantitative, e.g., requiring to keep some cost metric below a given threshold, they may need to adapt to changing requirements like changing cost thresholds under different conditions~\cite{jansen2018shielded}.
When an exhaustive computation of the shield for all possible state-action pairs is computationally infeasible, shields could be dynamically computed at runtime by only analyzing the relatively small set of forward reachable states upto a given horizon~\cite{konighofer2023online,filip23safety}.
Surprisingly, none of the existing works considered our setting of changing qualitative safety specifications.
While the shields based on forward reachable sets~\cite{konighofer2023online,filip23safety} could be used, their safety guarantee would be valid up to a fixed horizon only.
This is an important gap in the shielding literature, since runtime controller adaptation due to changing safety goals is an important topic in AI and control theory~\cite{majumdar2017funnel,fridovich2018planning,quan2021eva,nayak2023context}.

Our synthesis algorithms are powered by abstraction-based control (ABC), which is a collection of model-based synthesis algorithms for formally verified controllers of nonlinear and hybrid dynamical systems~\cite{tabuada2008approximate,reissig2016feedback,nilsson2017augmented}.
Our work uses a particular ABC algorithm based on feedback refinement relations~\cite{reissig2016feedback}, which uses a uniform grid-based abstraction of the system and offers a fast controller refinement process that is crucial for the fast runtime deployment of our shields.
In principle, our dynamic algorithm can be adapted to more advanced abstraction techniques, particularly the ones involving state space transformations in the pre-processing stage to achieve higher efficiency~\cite{brorholt2024efficient}.

\section{Preliminaries}

\noindent\textbf{Notation.}
Given an alphabet $X$, we will write $X^*$ and $X^\omega$ to respectively denote the set of finite and infinite sequences over $X$, and will write $X^\infty$ to denote $ X^*\cup X^\omega$.
Given a set $S\subseteq X^\infty$, we will write $\pref(S)$ to denote the set of every finite prefix of $S$, i.e., $\pref(S)\coloneqq \set{w\in X^*\mid \exists w'\in X^\infty\;.\; ww'\in S}$.

\medskip
\noindent\textbf{Control systems.} 
We consider \emph{continuous-state, discrete-time control systems}, described as tuples of the form $\tup{\Xs,\Us,\Ws,\f}$, where 
$\Xs\subset \mathbb{R}^n$, $\Us\subset\mathbb{R}^m$, and $\Ws\subset\mathbb{R}^p$ are all compact sets respectively called the \emph{state space}, the \emph{control input space}, and the \emph{disturbance input space}, and the function $\f\colon \Xs\times\Us\times \Ws\to \Xs$ is called the \emph{transition function}.
The constants $n$, $m$, and $p$ are all positive integers and are called the \emph{dimensions} of the respective spaces.

The semantics of the control system $\Sys = \tup{\Xs,\Us,\Ws,\f}$ is described using its transitions and trajectories.
For any given state $\x\in\Xs$, control input $\u\in\Us$, and disturbance input $\w\in \Ws$ of $\Sys$ at a given time step, the new state at the next time step is given by $\x' = f(\x,\u,\w)$, and we will express this as the \emph{transition} $\x\tran{\u,\w}\x'$
The \emph{trajectory} $\traj$ of $\Sys$ starting at a given \emph{initial} state $\x_0\in\Xs$ and caused by control and disturbance input sequences $\u_0,\u_1,\ldots$ and $\w_0,\w_1,\ldots$ is a sequence of transitions $\x_0\tran{\u_0,\w_0}\x_1\tran{\u_1,\w_1}\x_2\ldots$.
The sequence of states $\x_0,\x_1,\ldots \in \Xs^\infty$ appearing in the trajectory $\traj$ will be called the \emph{path} of $\traj$.
Trajectories and paths can be either finitely or infinitely long.

\medskip
\noindent\textbf{Controllers.}
Let $\Sys = \tup{\Xs,\Us,\Ws,\f}$ be a control system.
A \emph{controller} of $\Sigma$ is a \emph{partial} function of the form $\Xs^*\to2^\Us$, which determines the set of allowed control inputs given the history of past states at each point in time.
Every controller $\Cont$ of $\Sys$  produces a set of paths from a given initial state $\x_0\in\Xs$, defined as 
\begin{multline*}
\paths{\Sys}{\Cont}{\x_0}\coloneqq \Big\{ \x_0\x_1\ldots\in\Xs^\infty\Big\vert \exists \w_0\w_1\ldots\in\Ws^\infty\;.\; \\ 
 \x_0\tran{u_0,\w_0}{\x_1}\tran{u_1,\w_1}{\x_2}\ldots \text{ is a trajectory of } \Sys \text{ where } \forall i\geq 0\;.\; u_i\in \Cont(x_0\ldots x_i) \Big\}.
\end{multline*}

We will encounter the following three orthogonal subclasses of controllers, where each subclass can be combined with other subclasses.
\begin{itemize}
	\item A \emph{state-feedback (memoryless) controller} is a controller $\Cont$ that only considers the current state while selecting control inputs.
	Formally, $\Cont(y)=\Cont(y')$ for every $y,y'$ whose last states are the same.
	We will represent state-feedback controllers as functions of the form $\Cont\colon\Xs\to 2^\Us$, whose domain is defined as the set of every $\x\in \Xs$ for which $\Cont(\x)$ is defined, and written as $\dom{\Cont}$.
	\item A \emph{deterministic controller} is a controller $\Cont$ that selects a single control input at each step, i.e., has the form $\Xs^*\to \Us$.
	\item A \emph{nonblocking controller} is a controller $\Cont$ if its every generated finite path has an infinite extension, i.e., $\paths{\Sys}{\Cont}{\x_0}\cap \Xs^*\subseteq \pref\left(\paths{\Sys}{\Cont}{\x_0}\cap\Xs^\omega\right)$.
	In other words, every nonblocking controller $\Cont$ must disallow every control input $u$ for every finite path $x_0\ldots x_k$ if there exists a disturbance $w\in\Ws$ with $x_{k+1}=\f(x_k,u,w)$ such that $\Cont(x_0\ldots x_k x_{k+1})$ is undefined.
\end{itemize}

\medskip
\noindent\textbf{Safety specifications.}
Let $\Sys = \tup{\Xs,\Us,\Ws,\f}$ be a control system, and let $G\subseteq \Xs$ be a set of states, designated as the set of \emph{safe} states.
The complement of the safe states will be called the \emph{unsafe} states.
The \emph{safety specification} (with respect to $\Sys$ and $G$) is the set of every sequence of states of $\Sys$ that never leaves $G$, formally written as $\safety_\Sys(G)\coloneqq \set{\path=\x_0\x_1\ldots\in\Xs^\infty \mid \forall i\geq 0\;.\; \x_i\in G}$.
When the system is clear from the context, we will drop the suffix and write $\safety(G)$.

\medskip
\noindent\textbf{Safety controllers.}
Let $\Sys = \tup{\Xs,\Us,\Ws,\f}$ be a control system and $\safety(G)$ be a safety specification.
A state-feedback controller $\Cont$ of $\Sys$ is called a \emph{safety controller} for $\safety(G)$, if, intuitively, $\Cont$ guarantees that all paths of the system stay forever inside $G$ no matter what disturbance inputs are experienced; formally, $\Cont$ must fulfill $\paths{\Sys}{\Cont}{x_0}\subseteq \safety(G)$ for every $x_0\in \dom{\Cont}$.
It is known that for fulfilling safety specifications of the form $\safety(G)$, state-feedback controllers suffice~\cite{vidal2000controlled}.
It is also easy to see that $\dom{\Cont}\subseteq G$.
From now on, we will denote a safety controller for $\safety(G)$ using $\Cont_G$. 

We add one final subclass of controllers to the list of other subclasses presented earlier.
For this, we say a controller $\Cont'$ is a \emph{sub-controller} of $\Cont$, written $\Cont'\sqsubseteq\Cont$, if (a)~$\dom{\Cont'}\subseteq \dom{\Cont}$ and
(b)~for every state $\x\in\dom{\Cont'}$, $\Cont'(x)\subseteq\Cont(\x)$.
Equivalently, we say $\Cont$ is the \emph{super-controller} of $\Cont'$.
\begin{itemize}
	\item For a given safety specification $\safety(G)$, a \emph{maximally permissive safety controller} is a safety controller $\Contmax_G$ such that every other safety controller $\Cont_G$ for $\safety(G)$ is a sub-controller of $\Contmax_G$, i.e., $\Cont_G\sqsubseteq \Contmax_G$.
\end{itemize}

It is known that if safety specifications admit controllers, then these controllers are \emph{unique} nonblocking, maximally permissive (and state-feedback) controllers~\cite{vidal2000controlled}, written \emph{\nmp controllers} in short.
Specifications other than safety (like liveness) lack this feature, though workarounds exist in the literature~\cite{anand2023synthesizing}.

\section{Dynamic Shielding for Parametric Safety Specifications}

\subsection{Preliminaries: The Existing (Safety) Shielding Framework}

We assume that the given control system needs to fulfill some functional tasks (like reaching a target in minimum time, etc.) while maintaining safety.
The functional tasks are served by a \emph{learned}\footnote{The term ``learned controller'' is a placeholder. Any unverified controller can be used.} black-box controller, and the safety constraints are enforced by the shield, which monitors the learned controller's decisions and overrides them if safety would be at risk otherwise.

\begin{definition}[Shields]
	Suppose $\Sys = \tup{\Xs,\Us,\Ws,\f}$ is a control system and $\safety(G)$ is a given safety specification.
	A \emph{shield} is a partial function $\Shield_G\colon \Xs\times \Us\to \Us'$ with $\Us'\subseteq \Us$, such that for every $\x\in\Xs$, $\Shield_G(\x,\u)$ is defined either for every $\u\in \Us$ or for none of $\u\in \Us$.
	The domain of the shield $\Shield_G$ is defined as: $\dom{\Shield_G} \coloneqq \set{\x\in\Xs\mid \Shield_G(\x,\u) \text{ is defined for all } \u\in\Us}$.
\end{definition}

Suppose the shield $\Shield_G$ is deployed with the learned controller $\Contfc\colon\Xs^*\to\Us$.
Let $\x$ be the current state at a given time point.
First, $\Contfc$ proposes the control input $\u=\Contfc(\ldots\x)$, and then, the shield $\Shield_G$ takes into account the pair $(\x,\u)$, and selects the control input $\u' = \Shield_G(\x,\u)$ that is possibly different from $\u$.
The set of resulting paths starting at a given initial state $\x_0\in\Xs$ is given as:
\begin{multline*}
\paths{\Sys}{\Contfc,\Shield_G}{\x_0}\coloneqq \Big\{ \x_0\x_1\ldots\in\Xs^\infty\,\Big\vert\, \exists \w_0\w_1\ldots\in\Ws^\infty\;.\;\\ 
 \x_0\tran{\Shield_G(\x_0,\Contfc(\x_0)),\w_0}{\x_1}\tran{\Shield_G(\x_1,\Contfc(\x_0\x_1)),\w_1}{\x_2}\ldots \text{ is a trajectory of } \Sys \Big\}.
\end{multline*}
The shield $\Shield_G$ guarantees safety under the learned controller $\Contfc$ from the initial state $\x_0$ if $\paths{\Sys}{\Contfc,\Shield_G}{\x_0}\subseteq \safety(G)$.

Whenever the output $u'$ of the shield $\Shield_G$ is different from the output of the controller $\Contfc$, we say that $\Shield_G$ has \emph{intervened}, and we want minimal interventions while fulfilling safety.
Formally, a shield is said to be \emph{minimally intervening} if every intervention is a necessary intervention, i.e., without the intervention, disturbances could push the trajectory outside of the shield's domain, and therefore safety guarantees would be lost.
Our definition of minimal intervention is adapted from the definition by Bloem et al.~\cite{bloem2015shield}, which formalizes minimal intervention with respect to a generic intervention-penalizing cost metric.

\begin{problem}[Minimally intervening shield synthesis]\label{prob:shield synthesis}\\
	\textit{Inputs:} A control system $\Sys$ and a safety specification $\safety(G)$.\\
	\textit{Output:} A shield $\Shield_G^*$ such that for every learned controller $\Contfc$ and for every $\x_0\in \dom{\Shield_G^*}$, the following hold:
	\begin{description}
		\item[Safety:] $\paths{\Sys}{\Contfc,\Shield_G^*}{\x_0}\subseteq \safety(G)$;
		\item[Minimal interventions:] for every finite path $\x_0\ldots \x_k\in \paths{\Sys}{\Contfc,\Shield_G^*}{\x_0}$, if an intervention happens, i.e., if $\Shield_G^*(\x_k,\Contfc(\x_0\ldots\x_k)) \neq \Contfc(\x_0\ldots\x_k)$, then there exists a $\w\in \Ws$ such that $f(\x,\Contfc(\x_0\ldots\x_k),\w)\notin\dom{\Shield_G^*}$.
	\end{description}
\end{problem}

The output of Problem~\ref{prob:shield synthesis} will be called the minimally intervening shield for $\Sys$ and $\safety(G)$, and can be obtained from \nmp safety controllers.

\begin{theorem}\label{thm:shielding via safety controller synthesis}
	Let $\Sys = \tup{\Xs,\Us,\Ws,\f}$ be a control system, $\safety(G)$ be a safety specification, and $\Contmax_G$ be the (unique) nonblocking, maximally permissive (\nmp) controller of $\Sys$ for $\safety(G)$. 
	Then, every minimally intervening shield $\Shield_G^*$ for $\Sys$ and $\safety(G)$ fulfills:
	\begin{align}\label{eqn:optimal shield from maximal controller}
		\Shield_G^*(x,u) =
			\begin{cases}
				u	&	\text{if } u\in \Contmax_G(x)\\
				u'\in \Contmax_G(x)	&	\text{otherwise}.
			\end{cases}
	\end{align}
	for every $(x,u)\in \Xs\times\Us$.
\end{theorem}

\begin{proof}
	By virtue of maximal permissiveness of $\Contmax_G$, we can infer that every $u\notin \Contmax_G(x)$ may potentially violate safety, since otherwise we could construct a safe super-controller of $\Contmax_G$ that would allow $u$ from $x$, and would otherwise mimic $\Contmax_G$.
	This is not possible since it would contradict the maximal permissiveness assumption of $\Contmax_G$.
	Since $\Shield_G^*$ needs to guarantee safety with its choice of control inputs, therefore it must select control inputs allowed by $\Contmax_G$.
	
	Now for the given $x,u$, if $u\in\Contmax_G(x)$ but $\Shield_G^*(x,u)\neq u$, then the shield violates the minimal intervention requirement, since we know that selecting $u$ instead would not lead to a violation of safety.\qed
\end{proof}

Minimally intervening shields are not unique, since any $u'\in \Contmax_G(x)$ can be selected when $u\notin\Contmax_G(x)$ in Eqn.~\eqref{eqn:optimal shield from maximal controller}.
We will use the heuristics of selecting the $u'$ that minimizes the Euclidean distance from the original input $u$.
However, this does not provide any long-run optimality guarantees, and selecting the best intervening input is still an open problem in shield synthesis.

\begin{remark}
We consider the so-called post-shielding framework, where the shield operates alongside an already learned controller
In contrast, in the pre-shielding framework, shields are used already during the training phase of the controller to prevent safety violations.
It is known that safety shields---and by extension our dynamic safety shields---can be used in both pre and post settings~\cite{pranger2021tempest}, though we will only use the post-shielding view for simplicity.
\end{remark}

\subsection{Problem Statement}
Traditional shield synthesis problems closely match the statement of Problem~\ref{prob:shield synthesis}, implying that shields would need recomputation if safety specifications change.
This is problematic when the precise safety specification is unknown apriori, and the shield needs to adapt as new safety requirements are discovered during runtime.
To mitigate this, we propose a new dynamic variant of Problem~\ref{prob:shield synthesis}.

We formalize parametric safety specifications.
Suppose $\Sys = \tup{\Xs,\Us,\Ws,\f}$ is a control system, and $R$ is a finite set of subsets of $\Xs$, i.e., $R = \set{G_0,\ldots,G_l}\subset 2^\Xs$.
Each element of $R$ is called an \emph{atomic safe set}.
The \emph{parametric} safety specification $\psafety(R)$ on $R$ is the family of all safety specifications generated by the safe sets that are conjunctions of subsets of $R$, i.e., $\psafety(R)=\set{\safety(G)\mid \exists S\subseteq R\;.\;G=\cap_{S'\in 
S} S'}$.
Clearly, the size of $\psafety(R)$ is $2^{|R|}$.

A \emph{dynamic shield} for $R$ is a function mapping every safety specification $\safety(G)\in \psafety(R)$ to a regular, static shield for $\safety(G)$.
We assume that the set $R$ is provided \emph{statically} during the design of the dynamic shield, while the actual safety specification $\safety(G)\in \psafety(R)$ is chosen \emph{dynamically} during runtime.

\begin{problem}[Dynamic shield synthesis]\label{prob:parametric shield synthesis}\\
	\textit{Inputs:} A control system $\Sys = (\Xs,\Us,\Ws,\f)$ and a finite set $R\subset 2^\Xs$ of atomic safe sets.\\
	\textit{Output:} A dynamic shield $\Shield_R^*$ such that for every safety specification $\safety(G)\in \psafety(R)$, $\Shield_R^*(G)$ is a minimally intervening (static) shield for $\Sys$ and $\safety(G)$.
\end{problem}

Using Theorem~\ref{thm:shielding via safety controller synthesis}, Problem~\ref{prob:parametric shield synthesis} reduces to the problem of synthesizing dynamic safety controllers, which are functions that map every safety specification $\safety(G)\in \psafety(R)$ to an \nmp safety controller for $\safety(G)$.

\begin{problem}[Dynamic safety controller synthesis]\label{prob:parametric safety controller synthesis}\\
	\textit{Inputs:} A control system $\Sys = (\Xs,\Us,\Ws,\f)$ and a finite set $R\subset 2^\Xs$ of atomic safe sets.\\
	\textit{Output:} A dynamic safety controller $\Contmax_R$ such that for every safety specification $\safety(G)\in \psafety(R)$, $\Contmax_R(G)$ is a nonblocking, maximally permissive (static) safety controller for $\Sys$ and $\safety(G)$.
\end{problem}

Clearly, every solution to Problem~\ref{prob:parametric safety controller synthesis} leads to a solution to Problem~\ref{prob:parametric shield synthesis} via Eqn.~\eqref{eqn:optimal shield from maximal controller}.
Therefore, in what follows, we shift our focus to solving Problem~\ref{prob:parametric safety controller synthesis}.

\subsection{Efficient Dynamic Safety Controller Synthesis}
\label{sec:efficient adaptive safety control}

In theory, Problem~\ref{prob:parametric safety controller synthesis} can be solved using two different types of brute-force approaches:
The first one is a pure offline algorithm, where we iterate over the set of all safety specifications in $\psafety(R)$, and for each of them, compute an \nmp (static) safety controller.
The second one is a pure online algorithm, where we compute a new \nmp (static) safety controller after observing the current safety specification at each time point during runtime.
While the pure offline algorithm would be prohibitively expensive, owing to the exponential size of  $\psafety(R)$, the pure online algorithm would be expensive enough to be deployed during runtime, especially when controller latency could have detrimental effects.

Our dynamic algorithm strikes a balance between offline design and online adaptation, and proves to be significantly more efficient compared to both on their own.
Our new algorithm has an offline design phase and an online deployment phase.
During the offline design phase, we compute the \nmp safety controller $\Contmax_{G}$ for each atomic safety specification $\safety(G)$ for $G\in R$.
These controllers may be called the \emph{atomic} safety controllers.
During the online deployment phase, at each step the true safe set $G = G'\cap G''\cap\ldots$ is revealed, where $G',G'',\ldots\in R$, and the required safety controller for $\safety(G)$ is obtained by dynamically composing the corresponding atomic safety controllers $\Contmax_{G'},\Contmax_{G''},\ldots$.

The process of composing atomic safety controllers at runtime involves two steps, namely a \emph{product} operation, followed by enforcing \emph{non-blockingness}, both described below.
An illustration is provided in Figure~\ref{fig:illustration of composition}.

\begin{definition}[Controller product]\label{def:controller composition}
	Let $\Cont$ and $\Cont'$ be a pair of state-feedback controllers of a given control system $\Sys = \tup{\Xs,\Us,\Ws,\f}$.
	We define the product $\Cont$ and $\Cont'$ as the state-feedback controller $\Cont\bigotimes \Cont'$ such that for every $\x\in \dom{\Cont} \cap \dom{\Cont'}$, $\Cont\bigotimes \Cont'(x)=\Cont(x)\cap \Cont'(x)$, and for every $\x\notin \dom{\Cont} \cap \dom{\Cont'}$, $\Cont\bigotimes \Cont'(x)$ is undefined.
\end{definition}

Intuitively, the product controller $\Cont\bigotimes \Cont'$ outputs only those control inputs that are safe for both $\Cont$ and $\Cont'$, and suppresses those that are unsafe for at least one of them.
This, however, does not guarantee the nonblockingness of $\Cont\bigotimes \Cont'$ itself, as is illustrated in Figure~\ref{fig:illustration:product}, where the product controller $\Contmax_G\bigotimes \Contmax_H$ blocks at the state $e$.
Luckily, we will apply the product on the atomic safety controllers, which are \nmp, and it follows that all nonblocking safety controllers for the overall safety specification will be sub-controllers of the product.

\begin{figure}
	\centering
	\tikzset{every state/.style={minimum size=18pt},->/.style={-Latex}}
	\begin{subfigure}[t]{0.3\textwidth}
		\centering
		\pgfdeclarelayer{background layer}
		\pgfsetlayers{background layer,main}
		\def\d{0.4}
		\begin{tikzpicture}[node distance=0.7cm]
			\node[state,initial below]		(a)	at	(0,0)				{$a$};
			\node[state]		(b)	[left=of a]		{$b$};
			\node[state]		(c)	[above=of b]		{$c$};
			\node[state]		(d)	[below=of b]		{$d$};
			
			\node[state]		(e)	[right=of a]		{$e$};
			\node[state]		(f)	[above=of e]		{$f$};
			\node[state]		(g)	[below=of e]		{$g$};
			
			\path[->]
				(a)		edge	node[above]		{$u_1$}	(b)
				(b)		edge	node[right]		{$u_2$}			(d)
			;
						
			\path[->]
				(a)		edge	node[above]		{$u_2$}		(e)
				(b)		edge	node[right]		{$u_1$}			(c)
				(e)		edge	node[left]		{$u_1$}		(f)
						edge	node[left]		{$u_2$}	(g);
						
			\path[->]
				(c)		edge[loop above]	()
				(f)		edge[loop above]	()
				(d)		edge[loop below]	()
				(g)		edge[loop below]	();

			\begin{pgfonlayer}{background layer}
				\draw[fill=blue,opacity=0.3]	($(d)+(-\d,-\d)$)		--	
												($(g)+(\d,-\d)$) --
												($(e)+(\d,\d)$) -- 
												($(b)+(\d,\d)$) --
												($(c)+(\d,\d)$) --
												($(c)+(-\d,\d)$) --
												cycle;
												
				\draw[fill=red,opacity=0.3]	($(d)+(-\d,-\d)$)		--	
												($(d)+(\d,-\d)$) --
												($(b)+(\d,-\d)$) -- 
												($(e)+(\d,-\d)$) --
												($(f)+(\d,\d)$) --
												($(c)+(-\d,\d)$) --
												cycle;
			\end{pgfonlayer}
			
		\end{tikzpicture}
		\caption{$\dom{\Contmax_G}$ is in red and $\dom{\Contmax_H}$ is in blue.}
		\label{fig:illustration:atomic controllers}
	\end{subfigure}
	\begin{subfigure}[t]{0.3\textwidth}
		\centering
		\pgfdeclarelayer{background layer}
		\pgfsetlayers{background layer,main}
		\def\d{0.4}
		\begin{tikzpicture}[node distance=0.7cm]
			\node[state,initial below]		(a)	at	(0,0)				{$a$};
			\node[state]		(b)	[left=of a]		{$b$};
			\node[state]		(c)	[above=of b]		{$c$};
			\node[state]		(d)	[below=of b]		{$d$};
			
			\node[state]		(e)	[right=of a]		{$e$};
			\node[state]		(f)	[above=of e]		{$f$};
			\node[state]		(g)	[below=of e]		{$g$};
			
			\path[->]
				(a)		edge	node[above]		{$u_1$}	(b)
				(b)		edge	node[right]		{$u_2$}			(d)
			;
						
			\path[->]
				(a)		edge	node[above]		{$u_2$}		(e)
				(b)		edge	node[right]		{$u_1$}			(c)
				(e)		edge	node[left]		{$u_1$}		(f)
						edge	node[left]		{$u_2$}	(g);
						
			\path[->]
				(c)		edge[loop above]	()
				(f)		edge[loop above]	()
				(d)		edge[loop below]	()
				(g)		edge[loop below]	();

			\begin{pgfonlayer}{background layer}
				\draw[fill=blue!50!red,opacity=0.45]	($(d)+(-\d,-\d)$)		--	
												($(d)+(\d,-\d)$)	--
												($(b)+(\d,-\d)$) --
												($(e)+(\d,-\d)$) -- 
												($(e)+(\d,\d)$) --
												($(b)+(\d,\d)$) --
												($(c)+(\d,\d)$) --
												($(c)+(-\d,\d)$) --
												cycle;
			\end{pgfonlayer}
			
		\end{tikzpicture}
		\caption{Product construction: $\dom{\Contmax_G \bigotimes\Contmax_H}$}
		\label{fig:illustration:product}
	\end{subfigure}
	\begin{subfigure}[t]{0.3\textwidth}
		\centering
		\pgfdeclarelayer{background layer}
		\pgfsetlayers{background layer,main}
		\def\d{0.4}
		\begin{tikzpicture}[node distance=0.7cm]
			\node[state,initial below]		(a)	at	(0,0)				{$a$};
			\node[state]		(b)	[left=of a]		{$b$};
			\node[state]		(c)	[above=of b]		{$c$};
			\node[state]		(d)	[below=of b]		{$d$};
			
			\node[state]		(e)	[right=of a]		{$e$};
			\node[state]		(f)	[above=of e]		{$f$};
			\node[state]		(g)	[below=of e]		{$g$};
			
			\path[->]
				(a)		edge	node[above]		{$u_1$}	(b)
				(b)		edge	node[right]		{$u_2$}			(d)
			;
						
			\path[->]
				(a)		edge	node[above]		{$u_2$}		(e)
				(b)		edge	node[right]		{$u_1$}			(c)
				(e)		edge	node[left]		{$u_1$}		(f)
						edge	node[left]		{$u_2$}	(g);
						
			\path[->]
				(c)		edge[loop above]	()
				(f)		edge[loop above]	()
				(d)		edge[loop below]	()
				(g)		edge[loop below]	();

			\begin{pgfonlayer}{background layer}
				\draw[fill=blue!50!red,opacity=0.45]	($(d)+(-\d,-\d)$)		--	
												($(d)+(\d,-\d)$)	--
												($(b)+(\d,-\d)$) --
												($(a)+(\d,-\d)$) -- 
												($(a)+(\d,\d)$) --
												($(b)+(\d,\d)$) --
												($(c)+(\d,\d)$) --
												($(c)+(-\d,\d)$) --
												cycle;
			\end{pgfonlayer}
			
		\end{tikzpicture}
		\caption{Ensuring nonblockingness: The domain of the largest nonblocking sub-controller of $\Contmax_G \bigotimes\Contmax_H$.}
		\label{fig:illustration:nonblockingness}
	\end{subfigure}
	\caption{Illustration of composition of atomic safety controllers.
	The transition system represents a finite-state control system with two control inputs $u_1,u_2$ and no disturbance inputs.
	The nodes are the states and the arrows represent the transition function.
	Suppose there are two safety specifications $\safety(\mathsf{red})$ and $\safety(\mathsf{blue})$, where $\mathsf{red}=\set{a,b,c,d,e,f}$ and $\mathsf{blue}=\set{a,b,c,d,e,g}$.
	The colored regions represent the domains of the respective controllers (purple is the intersection), and each controller's output at a given state is the set of control inputs for which the next state is in the domain.
	}
	\label{fig:illustration of composition}
\end{figure}

\begin{theorem}~\label{thm:on the composition}
	Let $\Sys = \tup{\Xs,\Us,\Ws,\f}$ be a given control system, $\safety(G)$ and $\safety(H)$ be safety specifications, and $\Cont_G^*$ and $\Cont_H^*$ be the nonblocking, maximally permissive (\nmp) controllers for  $\safety(G)$ and $\safety(H)$, respectively.
	A nonblocking controller is a safety controller for $\safety(G\cap H)$ if and only if it is a nonblocking sub-controller of $\Contmax_{G}\bigotimes \Contmax_H$.
\end{theorem}

\begin{proof}
	First, observe that $\safety(G\cap H) = \safety(G)\cap \safety(H)$.
	This follows from the definition of safety specifications.
	
	\noindent\textbf{[If:]} Suppose $\Cont$ is a nonblocking sub-controller of $\Contmax_{G}\bigotimes \Contmax_H$.
	Since this is a sub-controller of $\Contmax_{G}\bigotimes \Contmax_H$, which outputs control inputs that are allowed by both $\Cont_G^*$ and $\Cont_H^*$, it follows that $\Cont$ satisfies both $\safety(G)$ and $\safety(H)$, and therefore it satisfies $\safety(G\cap H)$.
	(Moreover, $\Cont$ is already assumed to be nonblocking.)
	
	\noindent\textbf{[Only if:]} Suppose $\Cont$ is a nonblocking safety controller for $\safety(G\cap H)$.
	Therefore $\Cont$ is a nonblocking safety controller for both $\safety(G)$ and $\safety(H)$, separately.
	I.e., the control inputs selected by $\Cont$ fulfill both $\safety(G)$ and $\safety(H)$ simultaneously, and therefore they are also allowed by $\Contmax_{G}\bigotimes \Contmax_H$.
	Therefore, $\Cont$ is a (nonblocking, by assumption) sub-controller of $\Contmax_{G}\bigotimes \Contmax_H$.\qed
\end{proof}

Theorem~\ref{thm:on the composition} dramatically narrows down the search space for the sought \nmp safety controller for $\safety(G\cap H)$, which is stated as follows. 
%In particular, as the sought controller has to be nonblocking, it is now guaranteed to be a sub-controller of $\Contmax_G\bigotimes \Contmax_H$.
%The maximal permissiveness on the other hand will be guaranteed by selecting the specific nonblocking sub-controller $\Contmax_G\bigotimes \Contmax_H$ that is the super-controller of all other nonblocking sub-controllers of $\Contmax_G\bigotimes \Contmax_H$.
%We summarize this below.

\begin{corollary}\label{cor:nonblocking maximally-permissive controller synthesis}
	The nonblocking, maximally permissive (\nmp) safety controller $\Contmax_{G\cap H}$ for $\safety(G\cap H)$ fulfills: 
\begin{enumerate}[(a)]	
	\item $\Contmax_{G\cap H}$ is nonblocking, 
	\item $\Contmax_{G\cap H} \sqsubseteq \Contmax_{G}\bigotimes \Contmax_H$, and 
	\item for every nonblocking $\Cont$ with $\Cont\sqsubseteq \Contmax_{G}\bigotimes \Contmax_H$, $\Cont\sqsubseteq \Contmax_{G\cap H}$.
\end{enumerate}
\end{corollary}

The nonblocking  sub-controller of $\Contmax_{G}\bigotimes \Contmax_H$ fulfilling (b) and (c) in Corollary~\ref{cor:nonblocking maximally-permissive controller synthesis} will be referred to as the \emph{largest nonblocking sub-controller} of $\Contmax_{G}\bigotimes \Contmax_H$.

The computation of \nmp safety controllers and largest nonblocking sub-controllers may or may not be decidable, depending on the nature of the state space (finite or infinite) and the nature of the transition function (linear or nonlinear) of the system~\cite{vidal2000controlled}.
We will present a sound but incomplete abstraction-based synthesis algorithm in the next section.
For the moment, assuming largest nonblocking sub-controllers can be algorithmically computed, we summarize the overall dynamic safety controller synthesis algorithm in the following.

%\medskip
\begin{figure}
\hrule
\smallskip
\noindent\textbf{Inputs:} Control system $\Sigma$ and the set $R$ of atomic safe sets \\
\noindent\textbf{Output:} Dynamic safety controller $\Contmax_R$ for the parameterized safety specification $\psafety(R)$\\
\noindent\textbf{The algorithm:}
\begin{enumerate}[(A)]
\vspace{-0.3cm}
	\item \textbf{The offline design phase:} Compute the set of atomic safety controllers $\set{\Contmax_{G_i}}_{{G_i}\in R}$, where $\Contmax_{G_i}$ is the \nmp safety controller for $G_i$.
	\item \textbf{The online deployment phase:} Let $x$ be the current state and $G = G_1\cap \ldots \cap G_k$ be the current obstacle, where $G_i\in R$ for all $i\in [1;k]$.
	We need to output $\Contmax_R(G)(x)$, which is obtained by composing the atomic safety controllers $\Contmax_{G_1},\ldots,\Contmax_{G_k}$ using the following two-step process:
	\begin{enumerate}[1.]
		\item  Compute the product $\Cont \coloneqq \Contmax_{G_1}\bigotimes \ldots\bigotimes \Contmax_{G_k}$.
	 	\item Compute the largest nonblocking sub-controller $\Cont'$ of $\Cont$, and output $\Contmax_R(G)(x)\coloneqq\Cont'(x)$.
	 \end{enumerate}
\end{enumerate}
\hrule
\caption{The dynamic safety controller synthesis algorithm}
\label{fig:dynamic safety controller synthesis algorithm}
\end{figure}
%\smallskip

\section{Synthesis Algorithm using Abstraction-Based Control}
\label{sec:ABC}

We now present a sound but incomplete procedure for implementing the algorithm in Figure~\ref{fig:dynamic safety controller synthesis algorithm}.
Most of the results in this section are closely related to the existing works from the literature, but are adapted to our setting.

\subsection{Preliminaries: Abstraction-Based Control (ABC)}

Abstraction-based control (ABC) is a collection of controller synthesis algorithms, which use systematic grid-based abstractions of continuous control systems and performs synthesis over these abstractions using automata-based approaches.
The strengths of ABC algorithms are in their expressive power, namely they support almost all widely used control system models alongside rich temporal logic specifications~\cite{majumdar2024symbolic,reissig2016feedback}.
Besides, ABC algorithms are usually implementable using efficient symbolic data structures, such as BDDs, helping us to devise efficient push-button controller synthesis algorithms in practice.

The typical workflow of an ABC algorithm has three stages, namely \emph{abstraction}, \emph{synthesis}, and (controller) \emph{refinement}.
We describe each stage one by one in the following, where we assume that we are given a control system $\Sys=\tup{\Xs,\Us,\Ws,\f}$ and a generic specification $\Phi\subseteq \Xs^\infty$ as inputs, and the aim is to compute a controller $\Cont_\Phi\colon\Xs\to 2^\Us$ whose domain is as large as possible such that $\paths{\Sys}{\Cont_\Phi}{\x_0}\subseteq \Phi$ for all $\x_0\in\dom{\Cont_\Phi}$.

\smallskip
\noindent\textbf{Abstraction.}
In the abstraction stage, the given control system $\Sys$ is approximated by a finite grid-based \emph{abstraction}.
There are many alternative approaches to construct the abstraction, and we use the one based on \emph{feedback refinement relations} (FRR)~\cite{reissig2016feedback}. 
In FRR, the abstraction is modeled as a separate control system $\Sysh=\tup{\Xsh,\Ush,\fh}$ without disturbances, where $\Xsh$ and $\Ush$ are finite sets, and $\fh$ is a nondeterministic transition function, i.e., has the form $\Xsh\times\Ush\to 2^{\Xsh}$.
The set $\Xsh$ is obtained as the collection of the finitely many grid cells created by partitioning the continuous state space $\Xs$; therefore, every element of $\Xsh$ is a subset of $\Xs$.
The set $\Ush$ is obtained as the collection of finitely many usually equidistant points selected from $\Us$, i.e., $\Ush$ is a finite subset of $\Us$.

Suppose $Q\colon\x\mapsto\xh$  with $\x\in\xh$ is a mapping that maps every continuous state of $\Sigma$ to the (unique) cell of $\Xsh$ it belongs to.
We will extend $Q$ to map sets of states and sets of state sequences of $\Sys$ to their counterparts for $\Sysh$ in the obvious manner.
We say $Q$ is an FRR from $\Sys$ to $\Sysh$, written $\Sys\preccurlyeq_Q\Sysh$, if for every $\x\in \Xs$, for every $\uh\in \Ush$, and for every $w\in\Ws$, there exists $\xh'\in \fh\left(Q(\x),\uh\right)$ such that $(\f(\x,\uh,w),\xh')\in Q$.
We omit the details of how to construct $\fh$ such that $\Sys\preccurlyeq_Q\Sysh$ holds, and refer the reader to the original paper~\cite{reissig2016feedback}.

When $\Sys\preccurlyeq_Q \Sysh$, it is guaranteed that for every controller $\Cont\colon\Xs\to\Ush$ of the system $\Sys$ and for every initial state $\x_0$, $Q\left(\paths{\Sys}{\Cont}{\x_0}\right)\subseteq \paths{\Sysh}{\Cont}{Q(\x_0)}$. %, where $Q\left(\paths{\Sys}{\Cont}{\x_0}\right) $ is the set of projections of paths in $\paths{\Sys}{\Cont}{\x_0}$ via $Q$. %, i.e.,  $Q\left(\paths{\Sys}{\Cont}{\x_0}\right) =\set{\xh_0\xh_1\ldots\mid \exists x_0x_1\ldots\in \paths{\Sys}{\Cont}{\x_0}\;.\;\forall i\geq 0\;.\;\xh_i=Q(\x_i)}$.
In other words, the paths of the abstraction $\Sysh$ conservatively over-approximates (with respect to the mapping $Q$) the paths of the control system $\Sys$ under the same controller $\Cont$ and for every sequence of disturbance inputs.

\smallskip
\noindent\textbf{Synthesis.}
In the synthesis stage, first, the given specification $\Phi$ is conservatively abstracted to the specification $\Phih$ for $\Sysh$ such that $\Phih\subseteq Q(\Phi)$.
When $\Phi$ is a safety specification $\safety(G)$ for some $G\subseteq \Xs$, the abstract specification can be chosen as $\Phih = \safety_{\Sysh}(Q(G))$, i.e., the paths of $\Sysh$ which avoid $Q(G)$ at all time.

Next, we treat the abstraction $\Sysh$ as a two-player, turn-based adversarial game arena, where the controller player chooses a control input $\uh$ at each state $\xh$, while the environment player resolves the nondeterminism in $\fh(\xh,\uh)$.
The objective of the controller player is to come up with an abstract controller $\Conth_{\Phih}\colon\Xsh\to 2^{\Ush}$ such that no matter what the environment player does, the resulting sequence of states remains inside the set $\Phih$.
In the next subsection, we will describe the algorithm for finding such abstract controllers for safety specifications.

\smallskip
\noindent\textbf{Refinement.}
The (controller) refinement is the stage where the abstract controller $\Conth_{\Phih}$ of $\Sysh$ for $\Phih$ is mapped back to a concrete controller $\Cont_\Phi$ for the system $\Sys$, which amounts to simply defining $\Cont_\Phi(\x) \coloneqq \Conth_{\Phih}(Q(\x))$ for every $\x\in \dom{\Cont} = \cup_{\xh\in \dom{\Conth_{\Phih}}}\xh$.
By virtue of the FRR $Q$ between $\Sys$ and $\Sysh$, it is guaranteed that $\paths{\Sys}{\Cont_\Phi}{\x_0}\subseteq\Phi$ for all $\x_0\in \dom{\Cont_\Phi}$; in other words, $\Cont_\Phi$ is a sound controller of $\Sys$.
It is worthwhile to mention that such a simple refinement stage is one unique strength of FRR, since the other alternatives~\cite{tabuada2008approximate,pola2009symbolic} usually require a significantly more involved refinement mechanism.

\begin{remark}\label{rem:incompleteness of ABC}
Even though ABC produces sound controllers, it lacks completeness.
This means that sometimes we will not find a controller even if there exists one, and even if we find a controller, its domain could be smaller than the one with the largest possible domain that exists in reality.
This is unavoidable if we are uncompromising with soundness, since temporal logic control of nonlinear control systems is undecidable in general~\cite{henzinger2000robust}.
The side-effect of using ABC to solve Problem~\ref{prob:parametric safety controller synthesis} is that the maximal permissiveness guarantee can no longer be achieved, though our safety controllers will be maximally permissive with respect to the abstraction.
One way to improve the permissiveness would be to reduce the discretization granularity in the abstraction, though this will increase the computational complexity due to larger abstraction size.
\end{remark}

\subsection{ABC-Based Dynamic Safety Control}

We now present ABC-based procedure to implement the algorithm in Figure~\ref{fig:dynamic safety controller synthesis algorithm}.
For this, we fix the abstraction $\Sysh$ of the system $\Sys$, assuming $\Sys\preccurlyeq_Q\Sysh$ for a given FRR $Q$, and present our algorithms on $\Sysh$.
Using the standard refinement process of ABC, we will obtain an dynamic safety controller for $\Sys$.

\medskip
\noindent\textbf{Step A: Computing Atomic Safety Controllers.}
Suppose $G_i\subseteq \Xs$ be an atomic unsafe set of states of $\Sys$.
As described above, the abstract atomic safety specification $\safety_{\Sysh}(\Ghi)$ is the set of paths of $\Sysh$ that remain safe with respect to $\Ghi=Q(G_i) = \set{\xh\in\Xsh\mid \xh\cap G\neq \emptyset}$.
The respective \nmp abstract controller $\Conth_{\Ghi}$ (\nmp with respect to $\Sysh$) can be computed using a standard iterative procedure from the literature~\cite{vidal2000controlled} and presented using the function \SafetySynt in Algorithm~\ref{alg:safety control}.
\SafetySynt uses the set $S$ as an over-approximation of the set of states from which the safety specification can be fulfilled (aka, controlled invariant set).
Initially $S$ spans the entire state space $\Xsh$ of $\Sysh$ (Line~\ref{line:safety control a}).
Afterwards, the over-approximation $S$ is iteratively refined (the do-while loop) as states from the current $S$ are discarded owing to inability of fulfilling safety from them.
This is implemented using the $\cpre\colon 2^{\Xsh}\to 2^{\Xsh}$ operator defined as 
$
\cpre(S) \coloneqq \set{ \xh\in\Xsh\mid \exists \uh\in\Ush\;.\;\fh(\xh,\uh)\subseteq S }.
$
When no more refinement of $S$ is possible, we stop the iteration and extract the safety controller $\Conth_{\Ghi}$ (Line~\ref{line:safety control c}) as the one that keeps the abstract system inside $S$.
It is guaranteed that $\Conth_{\Ghi}$ is an \nmp safety controller of $\Sysh$ for $\safety_{\Sysh}(\Ghi)$, and that its refinement is a nonblocking safety controller for $\safety_{\Sys}(G_i)$.
Unfortunately, the maximal permissiveness is not guaranteed with respect to $\Sys$, as explained in Remark~\ref{rem:incompleteness of ABC}.

\begin{minipage}{0.45\textwidth}
	\begin{algorithm}[H]
	\caption{\SafetySynt}
	\label{alg:safety control}
	\begin{algorithmic}[1]
		\Require $\Sysh = \tup{\Xsh,\Ush,\fh}$, $\safety_{\Sysh}(\Ghi)$
		\Ensure Safety controller $\Conth_{\Ghi}$ of $\Sysh$
		\State $S \gets \Xsh$ \label{line:safety control a}
		\Do 
			\State $S_{\mathsf{old}}\gets S$
			\State $S\gets \cpre(S) \cap \Ghi$ \label{line:safety control b}
		\doWhile{$S\neq S_{\mathsf{old}}$}
		\State $\forall \xh\in S\;.\;\Conth_{\Ghi}(\xh) \gets \set{\uh\in\Ush\mid \fh(\xh,\uh)\subseteq S}$\label{line:safety control c}
		\State \Return $\Conth_{\Ghi}$
	\end{algorithmic}
\end{algorithm}
\end{minipage}
\hfill
\begin{minipage}{0.45\textwidth}
	\begin{algorithm}[H]
	\caption{\NBSynt}
	\label{alg:nonblocking control}
	\begin{algorithmic}[1]
		\Require $\Sysh = \tup{\Xsh,\Ush,\fh}$, $\Conth\colon \Xsh\to 2^{\Ush}$
		\Ensure Largest nonblocking sub-controller $\Conth'$ of $\Conth$
		\State $\Xsh'\gets \Xsh\cup \set{\bot}$
		\State Define $\fh'\colon \Xsh'\times\Ush\to 2^{\Xsh'}$: \label{line:nonblocking control a}
		\Statex $\forall \xh\in \Xsh\;.\;\forall \uh\in \Conth(\xh)\;.\;\fh'(\xh,\uh) \gets \fh(\xh,\uh)$
		\Statex $\forall \xh\in \Xsh\;.\;\forall \uh\notin \Conth(\xh)\;.\;\fh'(\xh,\uh) \gets \set{\bot}$
		\Statex $\forall \uh\in \Ush\;.\;\fh'(\bot,\uh) \gets \set{\bot}$
		\State $\Sysh'\gets \tup{\Xsh',\Ush,\fh'}$
		\State \Return $\SafetySynt(\Sysh',\Xsh = \Xsh'\setminus \set{\bot})$
	\end{algorithmic}
\end{algorithm}
\end{minipage}

\medskip
\noindent\textbf{Step B1: Computing the Product.}
Computing the product involves the straightforward application of Definition~\ref{def:controller composition} on two abstract safety controllers.

\medskip
\noindent\textbf{Step B2: Computing Largest Nonblocking Sub-Controllers.}
The largest nonblocking sub-controller is computed using the function \NBSynt, presented in Algorithm~\ref{alg:nonblocking control}.
\NBSynt first modifies $\Sysh$ to a new system $\Sysh'$ by keeping those transitions that are allowed by $\Conth$, and redirecting the rest to a new sink state $\bot$ (Line~\ref{line:nonblocking control a}).
With this modification, any safety controller of $\Sysh'$ that is nonblocking and avoids the unsafe state $\bot$ is by construction a nonblocking sub-controller of $\Conth$.
If the subcontroller is in addition maximally permissive with respect to the unsafe state $\bot$, then it follows that it is the largest nonblocking sub-controller of $\Conth$.
Therefore, the largest nonblocking sub-controller of $\Conth$ is obtained by invoking the subroutine \SafetySynt with arguments $\Sysh'$ and $\safety_{\Sysh'}(\Xsh)$.

In contrast, the na\"ive pure online shield synthesis algorithm would at each step run $\SafetySynt(\Sysh,\safety_{\Sys}(\Xsh))$, which is significantly more expensive than the procedure outlined above.
This is because in B2 (dominates B1), each invocation of $\cpre(\cdot)$ in \SafetySynt is significantly faster on $\Sysh'$ compared to invoking $\cpre(\cdot)$ on $\Sysh$ (the pure online case), as the complexity of $\cpre(\cdot)$ is linear in the number of transitions of the abstract system, and this number effectively becomes small for $\Sysh'$ since we can  ignore all the transitions that lead to ``$\bot$'' (\emph{surely} unsafe transitions).
The smaller effective number of transitions also contributes to a smaller number of iterations of the while loop in \SafetySynt, creating a compounding effect in reducing the overall complexity.

\subsection{Symbolic Implementation}

Our dynamic safety controller synthesis algorithm is implemented symbolically using BDDs, where the states and inputs and transitions of abstract control systems are modeled using boolean formulas represented by BDDs, and all the steps of Algorithm~\ref{alg:safety control} and \ref{alg:nonblocking control} and the product operation are implemented using logical operations over the BDDs and existential and universal quantifications.
The implementation details follow standard procedures used by ABC algorithms from the literature~\cite{rungger2016scots,majumdar2023flexible}. 
In particular, our tool is built within the ABC tool Mascot-SDS~\cite{majumdar2023flexible}, which supports efficient, parallelized BDD libraries like Sylvan~\cite{van2015sylvan}.
These implementation details enabled us to create a prototype dynamic shielding tool that whose offline computation stage takes a few minutes, and, more importantly, the online computations finish within just a few seconds on an average at each step.
More details on the experiments are included in Section~\ref{sec:experiments}.

\section{Dynamic Shields for Robot Navigation in\\ Unknown Territories}

\begin{wrapfigure}{r}{5.2cm}
	\vspace{-1cm}
	\centering
	\includegraphics[scale=0.2, clip, trim=0cm 4cm 5cm 0cm]{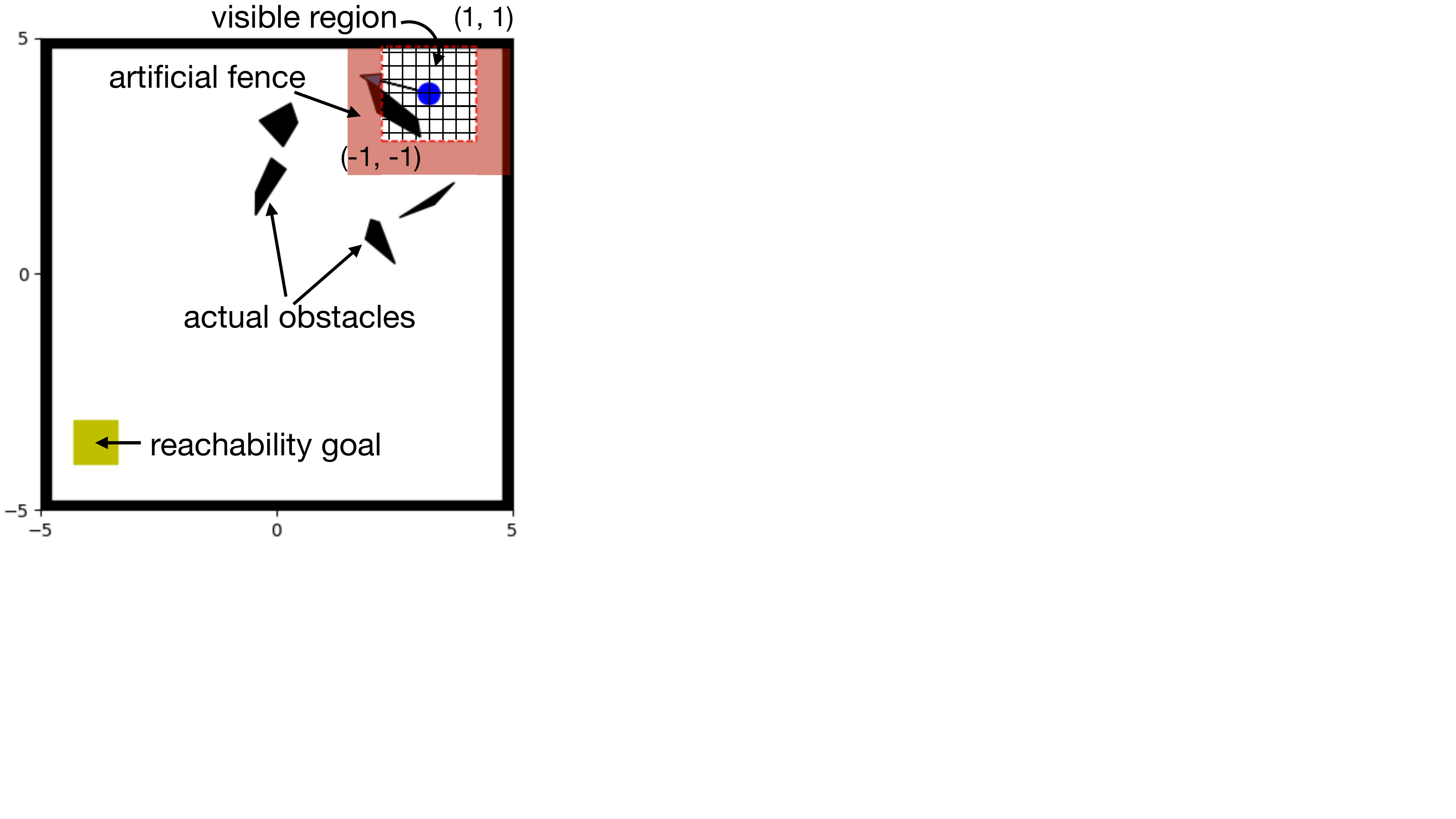}
	\vspace{-2cm}
	\caption{Illustration of dynamic shielding of the robot (blue dot) in an unknown environment. 
	Dynamic shields are computed using ABC, but not over the entire state space, rather over the the tiny visible region of the robot.
	The corners ``$(-1,-1)$'' and ``$(1,1)$'' of the visible region are in the robot's own reference coordinates. }
	\label{fig:dynamic shield illustration}
\vspace{-0.5cm}
\end{wrapfigure}
We consider a mobile robot placed in an unknown world filled with static obstacles (Figure~\ref{fig:dynamic shield illustration}).
The robot is controlled by an unknown AI controller with unknown motives.
We want to design a shield whose safety objective is to avoid colliding with the obstacles at all time.
We assume that the shield only has limited observation of the world, and it can only observe obstacles that are within a certain distance $d$ along each dimension of the X-Y coordinate axes.
This creates a \emph{visible region} that is a square whose sides have the length $2d$ centered around the current location of the robot at each time step.
This is a realistic scenario experienced by many mobile agents, including self-driving cars and exploratory robots.

The dynamic shield assumes that the robot's state space spans only the size of the visible region.
At the design time, the shield assumes that obstacles can be arranged in all possible ways within this visible region at each step.
At runtime, the shield observes the obstacles in the current snapshot of visible region, and does its online computations to quickly deploy the suitable shield.
This shield is deployed just for the current time step.
In the next step, the obstacle arrangements would have shifted, because the robot and its visible region could have moved, and therefore the shield must dynamically adapt and recompute the safe control inputs.
And the process repeats.

We need to ensure a \emph{safe handover} of two dynamically adapted shields at consecutive steps, i.e., every action that a shield allows must guarantee that the new state is in the domain of the shield at the next step.
The challenge comes from the lack of knowledge of obstacles that are currently outside of the visible region, but might suddenly appear in the next step and could seriously endanger safety.
We take a conservative approach.
We add \emph{artificial fences} in the outer periphery of the visible region, to over-approximate the uncertainty that awaits in the unobservable parts.
The safety specification for the shield at each step involves---other than avoiding the obstacles inside the visible region---avoiding the artificial fences all the time.
We formalize this as a dynamic shield synthesis problem.
The state space of the robot spans the visible region extended with the artificial fences of a given thickness $\epsilon>0$, i.e., $\Xs=[-d-\epsilon,d+\epsilon]\times [-d-\epsilon,d+\epsilon]$ with respect to the robot's own reference coordinate frame, in which the robot's initial state is always the origin.
We use the grid-based abstraction $\Xsh$ of $\Xs$, and assume that each element of $\Xsh$ is an atomic \emph{unsafe} state set.
This is because at runtime, no matter what obstacles are encountered within the visible region, they can be over-approximated as the union of the right subset of $\Xsh$. 
In addition, the fence is included in each atomic unsafe set.
Therefore, the set of atomic \emph{safe} sets is $\set{\Xs\setminus (\set{\xh}\cup \mathit{fence})\mid \xh\in\Xsh}$ where $\mathit{fence}=\Xs\setminus [-d,d]\times [-d,d]$.

\begin{figure}
	\input{FIG_handover_illustration}
	\caption{
		Illustration of the safe handover framework on a simplified discrete state space:
		Above is a sequence of configurations during the handover mechanism, where each configuration captures the state of the robot (blue), and the current locations of the unsafe states, including the fence fence (pink).
		The single square shaped state marked with ``$\star$'' represents the obstacle, and the arrows represent the only safe actions allowed by the shield at a given configuration.
%		In particular, the domains of the shields in (d) and (e) are empty due to the presence of the obstacle in the visible region (states surrounded by the fence).
%		This made the handover unsuccessful in (d) and (e), and the same shield from (a) was deployed during this time. 
%		In (f), the handover became successful.
	}
	\label{fig:handover illustration}
\end{figure}

Unfortunately, just adding artificial fences would not always lead to a safe handover.
This is illustrated in Figure~\ref{fig:handover illustration}, where for staying safe starting at the center of a given visible region, the shield needs to make a maneuver that requires all the states inside the visible region be safe (see configurations (a) and (f)). 
Starting at the configuration (a), the robot moves to its right---an action allowed by the shield.
But now the visible region includes the obstacle, causing the shield to have an empty domain (as shown in (d)).
Therefore, the new shield cannot be yet deployed and the handover would not be successful.
Luckily, the shield designed for (a) would still be sound, because, by construction, the robot would still stay within its domain.
This way, the shield designed at (a) remains in use until the new shield can be deployed, which happens in the configuration (f).

We formally present this procedure in Algorithm~\ref{alg:safe handover}, where, as before, instead of shields, we consider the handover problem for the equivalent n.m.p.\ safety controllers.
The subroutine $\mathrm{Init}$ takes as input the system dynamics in the visible region augmented with the fences, denoted as $\Sys$, as well as the initial safety specification $\safety(G)$, and checks if a shield is feasible or not for this initial configuration of the unsafe states.
Subsequently, as time progresses and the state changes, the subroutine $\mathrm{Next}$ is invoked, which takes as input the new safety specification $\safety(G')$ and the new state $x$ within the previous configuration, and it performs the safe handover operation by updating the global variables $\mathsf{config.\mhyphen in\mhyphen use}$, $\mathsf{state\mhyphen in\mhyphen use}$, and $\mathsf{cont.\mhyphen in\mhyphen use}$ that store which configuration and safety controller is currently in use.
The soundness is stated in Theorem~\ref{thm:soundness of safe handover}.

\renewcommand{\algorithmiccomment}[1]{\hfill{\color{black!50!white}$\triangleright$ #1}}
\begin{algorithm}
	\centering
	\caption{\texttt{SafeHandover}}
	\begin{minipage}{0.45\textwidth}
		\begin{algorithmic}[1]
%			\Require $\Sys$ \Comment{system dynamics within the visible region}
			\Function{$\mathrm{Init}$}{$\Sys$, $\safety(G)$}
				\State $\Sys \gets \Sys$ \Comment{global variable}
				\State Compute the n.m.p.\ cont.\ $\Contmax_G$ of $\Sys$ for $\safety(G)$
				\If{$(0,0)\in \dom{\Contmax_G}$}
%					\State $\cold \gets \langle\Sys,(0,0),\Shield_G^*\rangle$
					\State $\mathsf{config.\mhyphen in\mhyphen use} \gets \safety(G)$
					%\State $\mathsf{state\mhyphen in\mhyphen use} \gets (0,0)$
					\State $\mathsf{cont.\mhyphen in\mhyphen use} \gets \Contmax_G$
					\State \Return $\Contmax_G((0,0))$
				\Else \Comment{initially unsafe}
					\State \Return failure
%					\State $\cnext \gets \bot$
				\EndIf
			\EndFunction
		\end{algorithmic}
	\end{minipage}
%	\hfill
	\begin{minipage}{0.54\textwidth}
		\centering
		\begin{algorithmic}[1]
			\Function{$\mathrm{Next}$}{$\safety(G')$, $x$}
				\Statex \Comment{$x$: new state in the current $\mathsf{config.\mhyphen in\mhyphen use}$}
				\State Compute the n.m.p.\ cont.\ $\Contmax_{G'}$ of $\Sys$ for $\safety(G')$
				\If{$(0,0)\in \dom{\Contmax_{G'}}$} 
					\Statex \Comment{handover successful}
					\State $\mathsf{config.\mhyphen in\mhyphen use} \gets \safety(G')$
					\State $\mathsf{state\mhyphen in\mhyphen use} \gets (0,0)$
					\State $\mathsf{cont.\mhyphen in\mhyphen use} \gets \Contmax_{G'}$
				\Else \Comment{keep using the old shield}
					%\State $y\gets \mathsf{ChngCoord}(\mathsf{frame\mhyphen in\mhyphen use},x)$
					\State $\mathsf{state\mhyphen in\mhyphen use} \gets x$
				\EndIf
				\State \Return $\mathsf{cont.\mhyphen in\mhyphen use}(\mathsf{state\mhyphen in\mhyphen use})$
			\EndFunction
		\end{algorithmic}
	\end{minipage}
	\label{alg:safe handover}
\end{algorithm}

\begin{theorem}[Soundness of \texttt{SafeHandover}]\label{thm:soundness of safe handover}
	Let $\Sys = \tup{\Xs,\Us,\Ws,\f}$ be the control system representing the dynamics of the underlying system within the visible region, and let all obstacles (including the ones outside the visible region) be stationary.
	Suppose at a given time $t>0$, $x\in \Xs$ be the current state (in the visible region's reference frame), and $U\subseteq \Us$ be the output of $\mathrm{Next}$.
	Then for every $u\in U$ and for every $w\in \Ws$, the new state $x'=\f(x,u,w)$ is safe and, moreover, no matter what the new safety specification $\safety(G')$  at time $t+1$ is, the output of $\mathrm{Next}(\safety(G'),x')$ is non-empty.
\end{theorem}

\begin{proof}
	The safety of $x'$ follows from the property of the n.m.p.\ controller in the given visible region, and from the assumption that the obstacles are stationary.
	At time $t+1$, there are two possibilities: 
	(1)~The handover is successful, i.e., the n.m.p.\ safety controller $\Contmax_{G'}$ of $\Sys$ for $\safety(G')$ contains the origin $(0,0)$ in its domain, which coincides with $x'$ from the reference frame of the visible region at time $t$.
	In this case, the output of $\Contmax_{G'}((0,0))$ will by definition of ``domain'' be non-empty.
	(2)~The handover is unsuccessful, i.e., $(0,0)\notin \dom{\Contmax_{G'}}$.
	In this case, $\mathrm{Next}$ would output $\Contmax_{G}(x')$, where $\Contmax_G$ is the n.m.p.\ safety controller that was used at the previous time step $t$.
	It follows from the non-blockingness property of $\Contmax_{G}$ and from the stationarity assumption on the obstacles that $\Contmax_{G}(x')\neq \emptyset$.
\end{proof}

\section{Experiments}
\label{sec:experiments}

We implemented our algorithm within the tool Mascot-SDS, which can be accessed in the following url: https://gitlab.com/kmallik/mascotsds.

\medskip
\noindent\textbf{The Experimental Setup.}
The dynamics of the mobile robot is modeled as a discrete-time Dubins vehicle.
The system has three state variables $x$, $y$, and $\theta$, where $x$ and $y$ represent the location in the X-Y coordinate, and $\theta$ represents the heading angle in radians (measured counter-clockwise from the positive X axis); two control input variables $v$ and $a$, representing the forward velocity and the angular velocity of the steering; and three disturbance variables $w^1, w^2, w^3$, which affect the dynamics in the three individual states.
The transitions are:
\begin{align*}
	x' &= x + (v\cos\theta)\tau + w^1, \qquad
	y' &= y + (v\sin\theta)\tau + w^2, \qquad
	\theta' &= \theta + a\tau + w^3,
\end{align*}
where the primed variables on the left side represent the states at the next time step, and $\tau$ represents the sampling time.
We use the following spaces for the states, control inputs, and disturbance inputs:
$x\in [-1,1]$, $y\in [-1,1]$, $\theta\in [-\pi,\pi]$, $v\in \{-0.4,-0.2,\ldots,0.2,0.4\}$, $a\in \{-4,-3.5,\ldots,3.5,4\}$, $w^1\in [-0.01,0.01]$, $w^2\in [-0.01,0.01]$, and $w^3\in [-0.02,0.02]$.
Furthermore, we fix $\tau=0.1\,s$, and the thickness of the fence to $\epsilon=0.3$.

The underlying AI controller is generated using reinforcement learning (RL) with reach-avoid objectives.
In our experiments, the RL controller is made aware of the entire map, even though the shield's visibility range is limited to a tiny region ranging $[-1,1]\times [-1,1]$ in its own reference frame.

\medskip
\noindent\textbf{Performance Evaluations.}
We report the offline and online computation times of our dynamic shields for three different levels of abstraction coarseness used in the ABC algorithms.
The abstraction coarseness is measured as the (uniform) grid size used for discretizing the state and input spaces, which are respectively $3$ and $2$-dimensional vectors representing the dimension-wise side lengths of the square-shaped grid elements.
All synthesized shields are by-construction safe and minimally permissive, and therefore these aspects are not reported.
The code was run on a personal computer powered by Intel Core Ultra 7 255U processor and 32 GB RAM.

\begin{wraptable}{r}{8cm}
\vspace{-0.5cm}
	\caption{Computation times of the offline phase (atomic shield synthesis) of dynamic shields.}
	\label{table:offline computation times}
	\begin{tabular}{wc{2.5cm} wc{1.5cm}|wc{2cm} wc{1.5cm}}
	\toprule
		\multicolumn{2}{c|}{Grid size} & \multicolumn{2}{c}{Computation time} \\
		\midrule
		$\Xsh$ & $\Ush$	&	Abstraction	&	Synthesis \\
		\midrule
		$[0.10, 0.10, 0.30]$	&	$[0.2,0.5]$ & $2\, m\ 12\, s$ & $1\, m$ \\
		$[0.08, 0.08, 0.25]$	&	$[0.2,0.5]$ & $4\, m\,40\ s$ & $2\,m\ 35\,s$ \\
		$[0.06, 0.06, 0.20]$	&	$[0.2,0.5]$ & $9\, m\ 35\, s$ & $9\, m\ 25\, s$ \\
		\bottomrule
	\end{tabular}
\end{wraptable}
We report the offline computation times for the three different abstractions in Table~\ref{table:offline computation times}.
As expected, as the abstraction gets finer (smaller grid sizes), the computation time increases.
Nonetheless, all computation finished within reasonable amount of time.
In comparison, the pure offline baseline would timeout even for the coarsest abstraction, because with its X-Y state variables' grid sizes $[0.10,0.10]$, it would create $20\times 20=400$ grid cells in the domain $[-1,1]\times [-1,1]$, and since we choose the number of atomic safety specifications to be equal to  the number of grid cells, we would need to solve $2^{400}$ instances of safety controller synthesis problems!
Although the pure online baseline takes zero time in the offline phase, it will take more time at the online phase as we discuss next.

For each of the three abstraction classes, we deployed the pure online and dynamic shields alongside a learned controller and tested them on 70 randomly generated reach-avoid control problem instances.
In each instance and for each shield, we measured the computation time per step on an average, and report them in Figure~\ref{fig:experiments:timing comparison}.
We observe that the dynamic shields are almost always faster than the pure online shields, and as the abstraction gets finer, their difference becomes more prominent.
With the finest abstraction, the dynamic shield was upto five times faster!
Furthermore, any efficiency improvement of the pure online shield would benefit the dynamic shield too, because both rely on the \SafetySynt algorithm for their online computation phase.

\pgfplotsset{compat=1.13}
\begin{figure}
%	\resizebox{0.3\columnwidth}{!}{%
	\centering
	\begin{tikzpicture}
		\begin{axis}[xmin=0, xmax=3, ymin=0, ymax=3.5,scale=0.4,clip=false,
					ylabel=pure online ($s$), xlabel=dynamic ($s$)]
			\addplot+[only marks, mark color=blue] table [x=avg_computationAdaptive, y=avg_computationBaseline, col sep=comma] {PLOT_DATA/results_eta_1_1_3.csv};
			\draw	(axis cs: 0,0)	--	(axis cs: 3,3);
			\draw	(axis cs: 0,0)	--	(axis cs: 1.75,3.5);
			\draw	(axis cs: 0,0)	--	(axis cs: 1.1667,3.5);
			
			\node	at	(axis cs: 2.8,3.1)	{$\mathbf{1\times}$};
			\node	at	(axis cs: 1.75,3.8)	{$\mathbf{2\times}$};
			\node	at	(axis cs: 1.1667, 3.8)	{$\mathbf{3\times}$};
		\end{axis}
	\end{tikzpicture}
%	}
	\hfill
%	\resizebox{0.3\columnwidth}{!}{%
	\begin{tikzpicture}
		\begin{axis}[xmin=0, xmax=5, ymin=0, ymax=10,scale=0.4,clip=false,
					ylabel=pure online ($s$), xlabel=dynamic ($s$)]
			\addplot+[only marks, mark color=blue] table [x=avg_computationAdaptive, y=avg_computationBaseline, col sep=comma] {PLOT_DATA/results_eta_08_08_25.csv};
			\draw	(axis cs: 0,0)	--	(axis cs: 5,5);
			\draw	(axis cs: 0,0)	--	(axis cs: 5,10);
			\draw	(axis cs: 0,0)	--	(axis cs: 3.333,10);
			
			\node	at	(axis cs: 4.6,5.3)	{$\mathbf{1\times}$};
			\node	at	(axis cs: 4.8,10.7)	{$\mathbf{2\times}$};
			\node	at	(axis cs: 3.333, 10.7)	{$\mathbf{3\times}$};
		\end{axis}
	\end{tikzpicture}
%	}
	\hfill
%	\resizebox{0.3\columnwidth}{!}{%
	\begin{tikzpicture}
		\begin{axis}[xmin=0, xmax=10, ymin=0, ymax=15,scale=0.4,clip=false,
					ylabel=pure online ($s$), xlabel=dynamic ($s$)]
			\addplot+[only marks, mark color=blue] table [x=avg_computationAdaptive, y=avg_computationBaseline, col sep=comma] {PLOT_DATA/results_eta_06_06_2.csv};
			\draw	(axis cs: 0,0)	--	(axis cs: 10,10);
			\draw	(axis cs: 0,0)	--	(axis cs: 7.5,15);
			\draw	(axis cs: 0,0)	--	(axis cs: 5,15);
			\draw	(axis cs: 0,0)	--	(axis cs: 3.75,15);
			\draw	(axis cs: 0,0)	--	(axis cs: 3,15);
			
			\node	at	(axis cs: 9.3,8)	{$\mathbf{1\times}$};
			\node	at	(axis cs: 7.5,16)	{$\mathbf{2\times}$};
			\node	at	(axis cs: 5.3, 16)	{$\mathbf{3\times}$};
			\node	at	(axis cs: 3.75, 16)	{$\mathbf{4\times}$};
			\node	at	(axis cs: 2, 14)	{$\mathbf{5\times}$};
		\end{axis}
	\end{tikzpicture}
%	}
	\caption{Average online computation times of the pure online algorithm (the baseline) and our dynamic algorithm.
	Each point in the scatter plots represents one randomly generated problem instance for the navigation task.
	The three plots correspond to three different abstraction granularities used in the ABC procedure, with the leftmost plot representing the coarsest (Row 1, Table~\ref{table:offline computation times}) and the rightmost plot representing the finest (Row 3, Table~\ref{table:offline computation times}) abstraction sizes.}
	\label{fig:experiments:timing comparison}
\end{figure}

\section{Discussions and Future Work}

We propose dynamic shields that can adapt to changing safety specifications at runtime.
While the problem can be solved using pure offline or pure online approaches using brute force, our dynamic approach is significantly more efficient and combines both offline and online computations.
We presented concrete algorithms using the abstraction-based control framework, and demonstrated the effectiveness of dynamic shields on a robot motion planning problem.

Several future directions exist.
Firstly, in our work, we use the atomic safe set as it is given, and we will investigate if this set can be first processed in a way that the online deployment phase of shield computation can be benefited (e.g., by simplifying the nonblockingness process).
Secondly, in our simulations of the experiments, sometimes the system would get stuck and would not be able to make progress.
This is a known issue in shielding and we will study how to eliminate this by taking inspiration from other works dealing with similar problems~\cite{konighofer2023online}.
Thirdly, we proposed a conservative but simple approach to the safe handover problem, and more advanced procedures~\cite{nayak2023context} will be incorporated in subsequent versions.
Finally, we will extend our framework to richer classes of safety properties, expressible in LTL modulo theories~\cite{wu19shield,corsi24verification,rodriguez25shield,kim25realizable,rodriguez25explanations}, LTL over finite traces~\cite{giacomo2013finiteTraces,favorito23forward}, and quantitative or discounted safety specifications~\cite{almagor14discounting,khoury18tally}, as well as to control systems with dynamic obstacles for world models \cite{ha18world}. 

\begin{credits}
	\subsubsection{\ackname}
	This work was funded in part by the grant CEX2024-001471-M/funded by MICIU/AEI/10.13039/501100011033, and the DECO Project (PID2022-138072OB-I00) funded
by MCIN/AEI/10.13039/501100011033 and by the ESF+.
\end{credits}

 \bibliographystyle{splncs04}
 \bibliography{refs}
\end{document}